\documentclass[sigconf]{acmart}
\AtBeginDocument{%
  }

\usepackage{multirow}
\usepackage{xspace}
\usepackage[ruled,vlined]{algorithm2e}
\usepackage{makecell}

\copyrightyear{2026}
\acmYear{2026}
\setcopyright{cc}
\setcctype{by}
\acmConference[WWW '26]{Proceedings of the ACM Web Conference 2026}{April 13--17, 2026}{Dubai, United Arab Emirates.}
\acmBooktitle{Proceedings of the ACM Web Conference 2026 (WWW '26), April 13--17, 2026, Dubai, United Arab Emirates}
\acmPrice{}
\acmDOI{10.1145/3774904.3792491}
\acmISBN{979-8-4007-2307-0/2026/04}	%

\begin{document}

\title{MuVaC: A Variational Causal Framework for Multimodal Sarcasm Understanding in Dialogues}
\settopmatter{authorsperrow=4}
\author{Diandian Guo}
\orcid{0009-0002-8468-3285}
\affiliation{%
  \institution{Institute of Information Engineering, Chinese Academy of Sciences}
  \institution{School of Cyber Security, University of Chinese Academy of Sciences}
  \state{Beijing}
  \country{China}
}
\email{guodiandian@iie.ac.cn}

\author{Fangfang Yuan}
\affiliation{%
  \institution{Institute of Information Engineering, Chinese Academy of Sciences}
  \institution{School of Cyber Security, University of Chinese Academy of Sciences}
  \state{Beijing}
  \country{China}
}
\email{yuanfangfang@iie.ac.cn}

\author{Cong Cao}
\authornote{Corresponding author.}
\affiliation{%
  \institution{Institute of Information Engineering, Chinese Academy of Sciences}
  \institution{School of Cyber Security, University of Chinese Academy of Sciences}
  \state{Beijing}
  \country{China}
}
\email{caocong@iie.ac.cn}

\author{Xixun Lin}
\affiliation{%
  \institution{Institute of Information Engineering, Chinese Academy of Sciences}
  \state{Beijing}
  \country{China}
}
\email{linxixun@iie.ac.cn}

\author{Chuan Zhou}
\affiliation{%
  \institution{Academy of Mathematics and Systems Science, Chinese Academy of Sciences}
  \institution{School of Cyber Security, University of Chinese Academy of Sciences}
  \state{Beijing}
  \country{China}
}
\email{zhouchuan@amss.ac.cn}

\author{Hao Peng}
\affiliation{%
  \institution{Beihang University	}
  \state{Beijing}
  \country{China}
}
\email{penghao@buaa.edu.cn}

\author{Yanan Cao}
\affiliation{%
  \institution{Institute of Information Engineering, Chinese Academy of Sciences}
  \institution{School of Cyber Security, University of Chinese Academy of Sciences}
  \state{Beijing}
  \country{China}
}
\email{caoyanan@iie.ac.cn}

\author{Yanbing	Liu}
\affiliation{%
  \institution{Institute of Information Engineering, Chinese Academy of Sciences}
  \institution{School of Cyber Security, University of Chinese Academy of Sciences}
  \state{Beijing}
  \country{China}
}
\email{liuyanbing@iie.ac.cn}

\renewcommand{\shortauthors}{Diandian Guo et al.}

\begin{abstract}

The prevalence of sarcasm in multimodal dialogues on the social platforms presents a crucial yet challenging task for understanding the true intent behind online content.
Comprehensive sarcasm analysis requires two key aspects: Multimodal Sarcasm Detection (MSD) and Multimodal Sarcasm Explanation (MuSE).
Intuitively, the act of detection is the result of the reasoning process that explains the sarcasm.
Current research predominantly focuses on addressing either MSD or MuSE as a single task. 
Even though some recent work has attempted to integrate these tasks, their inherent causal dependency is often overlooked.
To bridge this gap, we propose MuVaC, a variational causal inference framework that mimics human cognitive mechanisms for understanding sarcasm, enabling robust multimodal feature learning to jointly optimize MSD and MuSE. 
Specifically, we first model MSD and MuSE from the perspective of structural causal models, establishing variational causal pathways to define the objectives for joint optimization.
Next, we design an alignment-then-fusion approach to integrate multimodal features, providing robust fusion representations for sarcasm detection and explanation generation.
Finally, we enhance the reasoning trustworthiness by ensuring consistency between detection results and explanations.
Experimental results demonstrate the superiority of MuVaC in public datasets, offering a new perspective for understanding multimodal sarcasm.
\end{abstract}



\maketitle

\section{Introduction}

Modern social media has transformed the internet into a vast conversational space~\cite{jovanovic2018multimodal}. Public discourse and personal interactions now unfold through dynamic, multimodal dialogues. 
Within these conversational threads, sarcasm is a pervasive and complex phenomenon. 
On video-based social platforms like YouTube and TikTok, sarcastic intent is no longer confined to static images or text. 
Instead, it is dynamically constructed through a complex interplay of dialogue history, vocal tone, facial expressions, and the shared context of a video.
Accurately understanding this nuanced, context-sensitive sarcasm is therefore critical for enhancing web applications, such as sentiment analysis \cite{maynard2014cares,farias2017irony}, public opinion mining \cite{cai2019multi} and the development of sophisticated web agents.

\begin{figure}[t]
    \centering
    \includegraphics[width=1\linewidth]{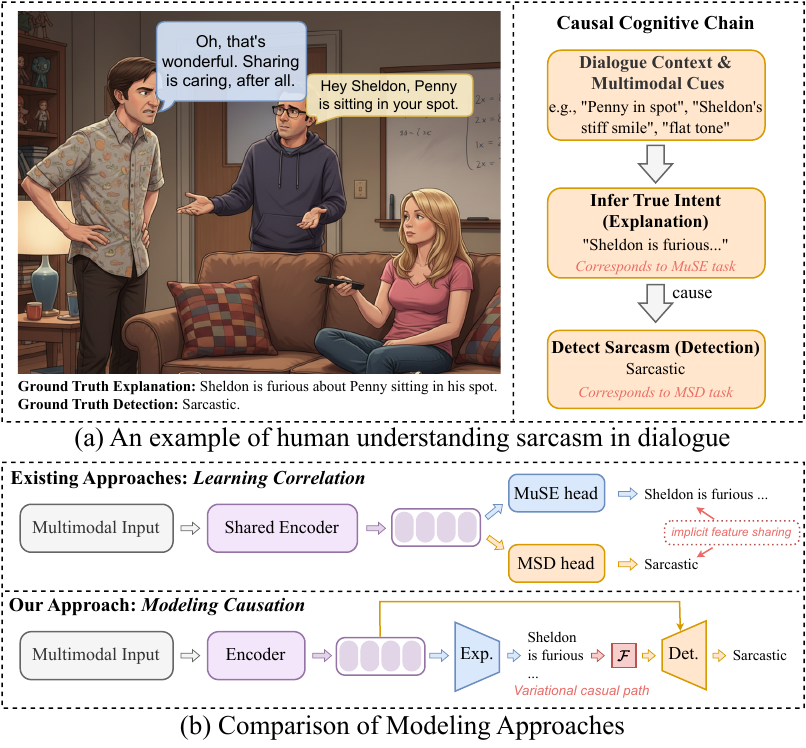}
    \caption{An illustration of the causal chain in human sarcasm understanding, contrasted with different paradigms.}%
    \label{fig:mot}
\end{figure}

Early research focuses primarily on identifying sarcasm in text and establishes a basic detection framework through lexical feature analysis and semantic pattern mining \cite{davidov2010semi, xiong2019sarcasm}. 
With the growth of multimodal content, sarcasm detection in dialogue requires a comprehensive analysis of multimodal factors such as textual content, gestures, and vocal tones, presenting significant technical challenges. 
This evolution has solidified two core research pillars for a comprehensive understanding of online sarcasm: Multimodal Sarcasm Detection (MSD), which aims to identify whether an utterance is sarcastic~\cite{Tomar2023YourTS,castro-etal-2019-towards,pandey2024vyang}, and Multimodal Sarcasm Explanation (MuSE)~\cite{desai2022nice,kumar2022did,kumar2023explaining}, which seeks to generate natural language to articulate the sarcastic intent.
While some studies~\cite{singh2024well,chen2024cofipara} attempt to unify these tasks via multi-task learning, they typically treat MSD and MuSE as parallel objectives with only implicit feature sharing.

However, these approaches overlook a fundamental insight: the relationship between sarcasm explanation and detection is not merely correlational but deeply causal~\cite{lin2025llm}.
From a cognitive science perspective~\cite{gibbs2007irony}, human sarcasm comprehension follows a distinct causal cognitive chain. 
As shown in Figure \ref{fig:mot}(a), we first integrate the conflicting cues to infer the speaker's true, non-literal intent (the \textit{explanation}).
It is only after identifying a sharp incongruity between this inferred intent and the literal meaning of the utterance that we subsequently identify it as sarcasm (the \textit{detection}).
In essence, identifying that a statement is sarcastic (MSD) is the result of understanding why it is sarcastic (MuSE).
However, this crucial cognitive process remains underexplored in existing methods.
As depicted in Figure~\ref{fig:mot}(b), current approaches learn to associate surface-level cues with a sarcasm label, but they lack the underlying reasoning of why those cues lead to a sarcastic interpretation. 
This reliance on "what is associated" rather than "what causes what" often leads to a lack of robustness and poor generalization on diverse web content.
Therefore, to build more intelligent and reliable systems, there is a clear need to shift the paradigm from correlation-based learning to explicitly causal modeling.

To address these challenges, we are inspired by the cognitive mechanism of human understanding of sarcasm and propose a \underline{\textbf{Mu}}ltimodal sarcasm understanding framework based on \underline{\textbf{Va}}riational \underline{\textbf{C}}ausal inference (MuVaC). 
Our core idea is to emulate the human cognitive process by establishing a causal pathway from explanation to detection. 
However, due to the unavailability of ground truth explanations during inference, modeling a direct causal path between MuSE and MSD would introduce biases, which contradicts our objectives. 
Therefore, we reformulate the structural causal model (SCM) into a deep variational inference framework with latent variables by incorporating hidden features to enable joint optimization, while ensuring causal relevance and consistency between detection and explanation.
To achieve controllable sarcasm explanations, we introduce probabilistic causal intervention during training to guide the generation of accurate explanations.
To foster robust feature representation for causal inference, we emphasize expressions and postures as crucial supplementary features during feature extraction.
For multimodal fusion, we propose an alignment-then-fusion approach to integrate multimodal features, providing a robust fused representation for causal inference.

Our contributions are as follows:

    $\bullet$  To the best of our knowledge, we are the first to formulate multimodal sarcasm understanding as a variational causal inference problem. 
    This paradigm shifts the objective from learning task correlations to modeling the cognitive causal chain, better aligning computational models with human reasoning.%
    
    $\bullet$  We propose MuVaC, an end-to-end framework that jointly optimizes MSD and MuSE through a unified causal architecture. MuVaC introduces a novel align-then-fusion module to learn robust multimodal representations for effective causal inference.%
    
    $\bullet$  %
    Experiments demonstrate that MuVaC not only achieves SOTA on both MSD and MuSE tasks, but also offers a new, causally-grounded perspective for multimodal sarcasm understanding. MuVaC achieves a significant F1-score improvement of nearly 10\% on the MUSTARD++ dataset and surpasses all strong baselines.

\begin{figure*}[t]
    \centering
    \includegraphics[width=1\linewidth]{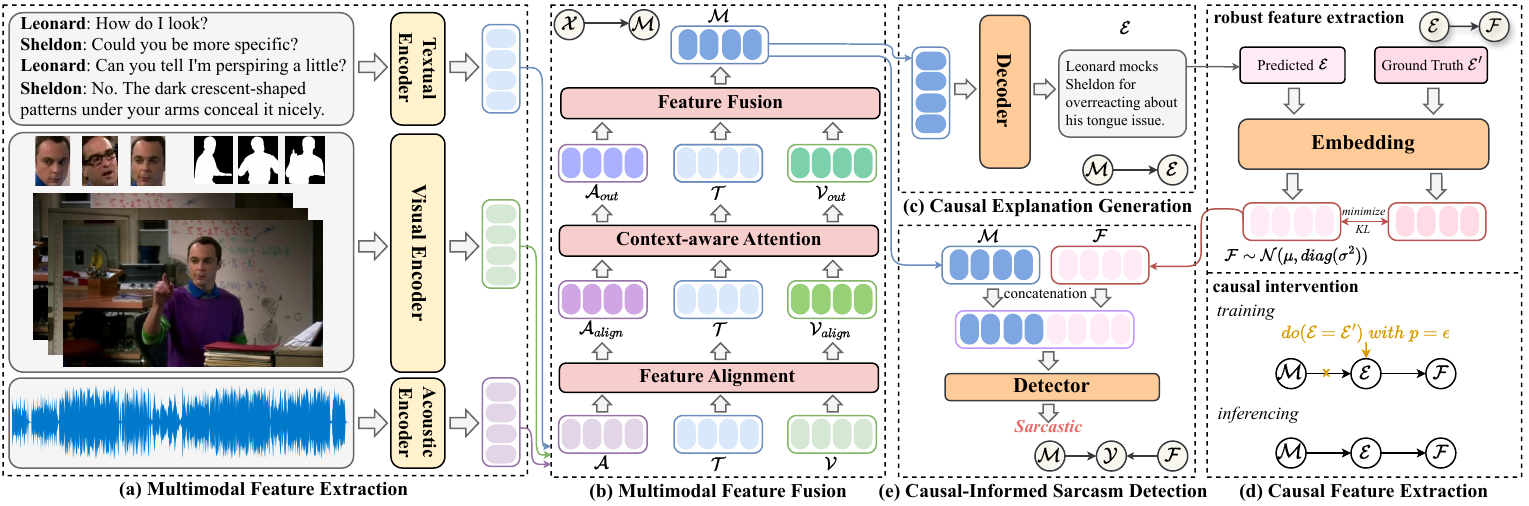}
    \caption{The overall architecture of MuVaC comprises five key steps: (a) Multimodal Feature Extraction, (b) Multimodal Feature Fusion,  (c) Causal Explanation Generation, (d) Causal Feature Extraction, and (e) Causal-Informed Sarcasm Detection. }
    \label{fig:model}
\end{figure*}

\section{Related Work}
\quad \textbf{\textit{Multimodal Sarcasm Detection.}}
Multimodal sarcasm detection refers to the process of identifying implicit sarcasm by analyzing multimodal contexts. 
\citet{castro-etal-2019-towards} developed the first multimodal conversational sarcasm detection benchmark dataset MUStARD. %
MO-Sarcation \citep{Tomar2023YourTS} verified the significant effect of the modal fusion order on the MSD task.
VyAnG-Net \citep{pandey2024vyang} extracted cross-modal discriminative features with adaptive ConvNets. 
MV-BART \citep{zhuang2024mv} dynamically adjusted the weights of multigranularity cues for different sarcastic scenarios.
Recent studies explored multitask learning frameworks to jointly optimize sarcasm detection with sentiment analysis \cite{chauhan-etal-2020-sentiment,bhosale2023sarcasm} and humor recognition \cite{Hasan_Lee_Rahman_Zadeh_Mihalcea_Morency_Hoque_2022}, offering insights for improving semantic understanding.
However, existing methods rely on simple combinations that overlook cross-task semantic correlations.
We address this by causally integrating MSD and MuSE, improving both credibility and interpretability.

\textbf{\textit{Multimodal Sarcasm Explanation.}}
Multimodal sarcasm explanation, as a derivative task of MSD, was first systematically defined by \citet{desai2022nice}.
\citet{kumar2022did} explored sarcasm analysis in dialogue scenarios and constructed the WITS dataset. 
In addition, they proposed MAF, using a new attention-based modal fusion mechanism as a solid baseline. 
TEAM \citep{jing2023multi} incorporated external common sense to enhance the generation of sarcasm explanations. 
SIRG \citep{singh2024well} introduced an innovative multimodal shared fusion method that seamlessly integrates cross-modal features through cross-attention mechanisms.
Although our approach aligns with the main MuSE paradigms, it uniquely integrates sarcasm explanation as a core component of the causal model to enhance sarcasm detection.

\textbf{\textit{Causal Inference.}}
Causal inference aims to explore the causal relationship between the observed variables rather than the superficial correlation by constructing SCMs~\cite{lin2025generative}.
Leveraging its capabilities in counterfactual reasoning and unbiased estimation, causal inference has been increasingly applied to various tasks such as computer vision \citep{yang2023context} and natural language processing \citep{wang2024debiased, xia2024aligning}.
Recently, causal inference has also been introduced into multimodal tasks. 
\citet{Agarwal_Shetty_Fritz_2020} propose automatic semantic image manipulation to alleviate spurious correlations during learning.
\citet{sun2022counterfactual} reveal harmful associations of text semantics in multimodal sentiment analysis through counterfactual manipulation.
\citet{Yang_Zhang_Qi_Cai_2021} propose a causal attention mechanism to reduce the confusion bias that misleads the attention module.
Different from existing works that focus on removing biased dependencies, we propose to reconstruct the neglected causal correlations between MuSE and MSD by transforming the SCM into deep variational inference.

\begin{figure}[t]
    \centering
    \includegraphics[width=\linewidth]{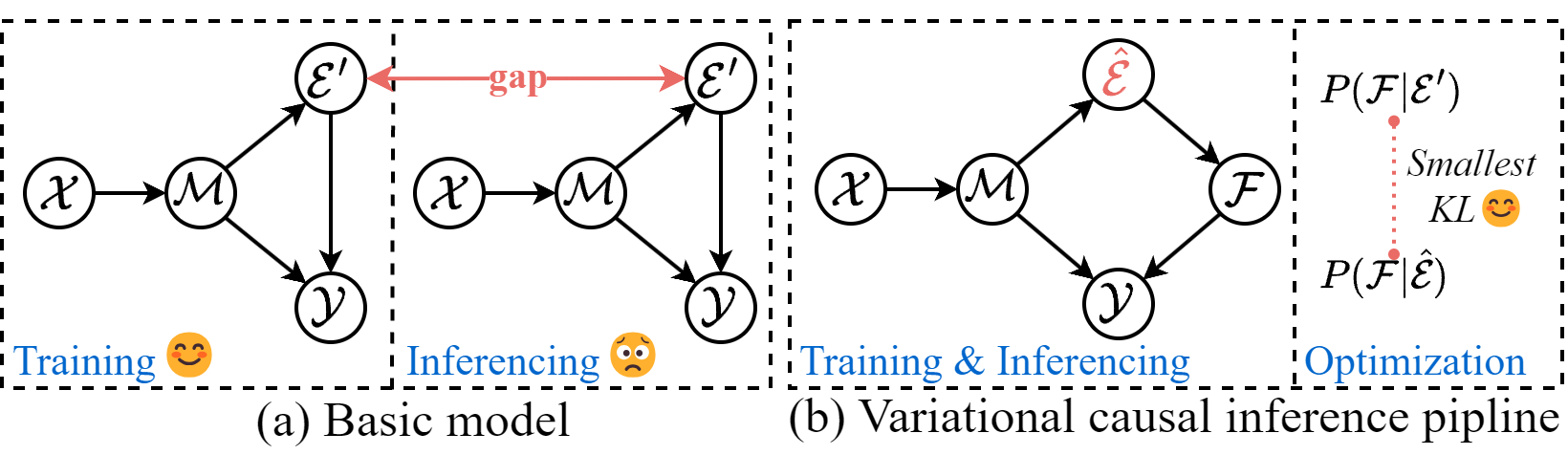}
    \caption{Structural causal models for MSD and MuSE.}
    \label{fig:vci}
\end{figure}

\section{Methodology}
The overall architecture of MuVaC is shown in Figure \ref{fig:model}. 
This section is organized as follows. We first formally define the joint task in Section \ref{3.1}. 
Following this, we introduce the theoretical foundation of our work: the variational causal framework in Section \ref{3.2}. 
We then proceed to detail the implementation of each component of the MuVaC architecture Section \ref{3.3}-\ref{3.8}, from multimodal feature processing to the final causal-informed detection.

\subsection{Task Definition}
\label{3.1}
The standard MSD task aims to detect whether an input $\mathcal{X} = (\mathcal{T}, \mathcal{V}, \mathcal{A})$ contains sarcasm and outputs $\mathcal{Y} \in \{0,1\}$. Here, $\mathcal{T}$, $\mathcal{V}$, and $\mathcal{A}$ denote the text, visual, and audio modalities, respectively.
Given an input $\mathcal{X}$, the traditional MuSE task is to generate an explanation $\mathcal{E}_{sar}$ to explain why $\mathcal{X}$ embodies sarcastic information. 
The task we aim to undertake involves not only producing a sarcasm detection result $\mathcal{Y}$ for $\mathcal{X}$ but also providing the corresponding explanation $\mathcal{E}$. 
This can be formalized as learning a model $g:  (\mathcal{T},\mathcal{V},\mathcal{A}) \rightarrow (\mathcal{Y},\mathcal{E})$.

\subsection{Variational Causal View for MuSE and MSD}
\label{3.2}

We study this problem from the perspective of SCM.
To establish and maintain consistency between explanation $\mathcal{E}$ and prediction $\mathcal{Y}$, a basic causal framework is illustrated in Figure \ref{fig:vci}(a). 
The optimization objective is $P(\mathcal{E}=\mathcal{E}', \mathcal{Y}=\mathcal{Y}' | \mathcal{M})$, where $\mathcal{E}'$ and $\mathcal{Y}'$ denote the ground-truth sarcasm explanation and label, respectively.
$\mathcal{M}$ is the multimodal feature.
During training, we can directly optimize the joint probability distribution using ground-truth explanations $\mathcal{E}'$:
$P(\mathcal{E}', \mathcal{Y}=\mathcal{Y}' | \mathcal{M}) = P(\mathcal{E}' | \mathcal{M}) P(\mathcal{Y}=\mathcal{Y}' | \mathcal{M}, \mathcal{E}')$. 
However, during inference, the basic model faces two critical challenges:
(1) It cannot access the ground truth explanation $\mathcal{E}'$; and (2) optimizing $P(\mathcal{Y}=\mathcal{Y}' | \mathcal{M}, \mathcal{E}=\mathcal{E}')$ becomes infeasible, as it is empirically difficult to generate an explanation that is entirely consistent with $\mathcal{E}'$.
Therefore, in this model, we can only use the explanation $\hat{\mathcal{E}}$ with the highest generative probability to approximate $P(\mathcal{Y}=\mathcal{Y}' | \mathcal{M},\hat{\mathcal{E}})$ during inference, where $\hat{\mathcal{E}}=\mathop{\arg\max}\limits_{\mathcal{E}} \; P(\mathcal{E}|\mathcal{M})$.

To mitigate the discrepancy between the ground-truth explanation $\mathcal{E}'$ and the generated explanation
$\hat{\mathcal{E}}$, we propose using a variational causal inference framework.
As shown in Figure \ref{fig:vci}(b), we introduce a front-door path $\mathcal{E}\rightarrow\mathcal{F}\rightarrow\mathcal{Y}$ compared to the base model in Figure \ref{fig:vci}(a). 
We define $\mathcal{F}\sim\mathcal{N}(\boldsymbol{\mu},diag(\boldsymbol{\sigma}^2))$ as the latent variable.
This allows us to quantify the explanation divergence between $P(\mathcal{Y}|\mathcal{M},\mathcal{E}')$ and $P(\mathcal{Y}|\mathcal{M},\hat{\mathcal{E}})$ through the KL-divergence $KL\big( P(\mathcal{F}|\mathcal{E}')  || P(\mathcal{F}|\hat{\mathcal{E}}) \big)$.
Similarly to the variational autoencoder, our objective is to minimize the distributional discrepancy between our model $p$ and the basic model $q$ by optimizing the Evidence Lower Bound (ELBO). 
This approach allows us to indirectly optimize the objective $\log p(\mathcal{Y}'|\mathcal{M})$:
\begin{equation}
\resizebox{0.9\linewidth}{!}{$
\begin{aligned}
\mathcal{L}_{cls}&=\log p(\mathcal{Y}'|\mathcal{M}) \\
&\geq 
\begin{aligned}
\mathbb{E}_{q(\mathcal{F}|\mathcal{M})}[\log p(\mathcal{Y}'|\mathcal{M},\mathcal{F})+\log p(\mathcal{F}|\mathcal{M})-\log q(\mathcal{F}|\mathcal{M})]\end{aligned} \\
&= 
\begin{aligned}
\mathbb{E}_{q(\mathcal{F}|\mathcal{M})}[\operatorname{log}p(\mathcal{Y}'|\mathcal{M},\mathcal{F})]-KL(q(\mathcal{F}|\mathcal{M})\parallel p(\mathcal{F}|\mathcal{M})),
\end{aligned} \\
\end{aligned}
$}
\end{equation}
where $q(\mathcal{F}|\mathcal{M})=\sum_\mathcal{E} q(\mathcal{F}|\mathcal{E})q(\mathcal{E}|\mathcal{M})=q(\mathcal{F}|\mathcal{E}')$, since $q(\mathcal{E}|\mathcal{M})$ is a Dirac delta distribution satisfying $q(\mathcal{E}'|\mathcal{M}) = 1$. However, $p(\mathcal{F}|\mathcal{M})= \sum_\mathcal{E} p(\mathcal{F}|\mathcal{E})p(\mathcal{E}|\mathcal{M})$ is difficult to compute. We use the generated explanation $\hat{\mathcal{E}}$ to approximate $p(\mathcal{F}|\mathcal{M})$, so the loss function can be expressed as:
\begin{equation}
\resizebox{0.9\linewidth}{!}{$
\mathcal{L}_{cls} = \underbrace{-\mathbb{E}_{q(\mathcal{F}|\mathcal{E}')}[\operatorname{log}p(\mathcal{Y}'|\mathcal{M},\mathcal{F})]}_{\text{reconstruction loss}}+\underbrace{KL(q(\mathcal{F}|\mathcal{E'})\parallel p(\mathcal{F}|\hat{\mathcal{E}}))}_{\text{latent divergence}},
$}\label{eq:clo}
\end{equation}
where the classification loss $\mathcal{L}_{cls}$ is composed of two critical terms. 
The first term, $ -\mathbb{E}_{q(\mathcal{F}|\mathcal{E}')}[\operatorname{log}p(\mathcal{Y}'|\mathcal{M},\mathcal{F})] $, is the standard cross-entropy loss for the detection task, which encourages the model to predict the correct sarcasm label based on both the multimodal features $\mathcal{M}$ and the explanation features $\mathcal{F}$.
The second term,  $KL(q(\mathcal{F}|\mathcal{E'})\parallel p(\mathcal{F}|\hat{\mathcal{E}}))$, is the cornerstone of our causal modeling. 
It acts as a regularizer, forcing the distribution of features $p(\mathcal{F}|\hat{\mathcal{E}})$ derived from the model-generated explanation $\hat{\mathcal{E}}$ to be as close as possible to the distribution of features $q(\mathcal{F}|\mathcal{E'})$ derived from the ground-truth explanation $\mathcal{E}'$. 
By minimizing this divergence, we ensure that the explanations our model generates are not just superficially plausible, but also contain the same essential causal information as the ground-truth explanations, thus making the subsequent detection process more robust and causally grounded.

Simultaneously, our other goal is to maximize the generation~probability $p(\mathcal{E}'|\mathcal{M} )$ of the ground-truth explanation $\mathcal{E}'$ by optimizing:
\begin{equation}
\mathcal{L}_{exp}=-\log p (\mathcal{E}'|\mathcal{M}).\label{eq:elo}
\end{equation}
A more detailed derivation can be found in the Appendix.

\subsection{Multimodal Feature Extraction}
\label{3.3}

This module extracts discriminative cross-modal features for sarcasm detection and explanation. 
Each modality is processed with tailored strategies to capture sarcastic cues.

\textit{\textbf{Textual Features.}}
We employ a pre-trained BART \cite{lewis2019bart} model as both the text encoder and the model backbone. 
By integrating BERT’s bidirectional contextual understanding with GPT’s autoregressive generation capability, BART effectively supports both the sarcasm detection task and explanation generation tasks.
For text input $\mathbf{T} = \{t_1,t_2,...,t_n\}$, BART extracts text features as $\mathcal{T}\in \mathbb{R}^{n\times d}$, where $n$ denotes the sequence length and $d$ represents the hidden layer dimension.

\textit{\textbf{Visual Features.}}
We take advantage of a pre-trained CLIP \cite{radford2021learning} model to extract visual features.
To address redundant static frames in videos and capture nonverbal cues related to sarcasm, we first perform action-aware keyframe sampling based on inter-frame similarity, reducing the temporal redundancy of the video. 
Then, we extract facial and pose features as auxiliary visual features.
The final visual representation is expressed as $\mathcal{V} = [\mathcal{V}_h;\mathcal{V}_f;\mathcal{V}_p] \in \mathbb{R}^{v\times d}$, where $\mathcal{V}_h$, $\mathcal{V}_f$, and $\mathcal{V}_p$ denote holistic, facial, and pose characteristics, respectively. $v$ is the sampled frame of $\mathcal{V}$.

\textit{\textbf{Acoustic Features.}}
For audio input $\mathbf{A}$, we sample the audio at a rate of 16 kHz, truncate the segments to equal length and extract the audio features using a pre-trained CLAP \cite{wu2023large} as $\mathcal{A} \in \mathbb{R}^{1\times d}$.

\subsection{Multimodal Feature Fusion}
\label{3.4}

The multimodal feature fusion corresponds to the path $\mathcal{X} \rightarrow \mathcal{M}$ in our SCM. 
Since MuSE tasks require natural language output, we employ text as the primary modality while using audio and vision as supplementary information. 
To ensure robust feature representation for both explanation generation and sarcasm detection, we perform multimodal feature fusion using adapter-like align-then-fusion (ATF) layers inserted into BART's encoder. 

\textit{\textbf{Feature Alignment.}}
Given multimodal features $\mathcal{X}=\{\mathcal{T}, \mathcal{A}, \mathcal{V}\}$, we first align visual and acoustic features with text features through cross-attention layers to project them into a shared vector space:
\begin{equation}
 \mathcal{V}_{align}= \mathrm{CA}(\mathcal{V},\mathcal{T}),\;%
\mathcal{A}_{align}= \mathrm{CA}(\mathcal{A},\mathcal{T}),%
\end{equation}
where $\mathrm{CA(\cdot,\cdot)}$ is a cross-attention layer. 

\textit{\textbf{Context-aware Attention.}}
We then guide aligned visual and acoustic features $\mathcal{V}_{align}$ and $\mathcal{A}_{align}$ as contextual information through context-aware self-attention layers to obtain $\mathcal{V}_{con}$ and $\mathcal{A}_{con}$:
\begin{gather}
\begin{bmatrix}
\hat{K} \\
\hat{V}
\end{bmatrix}=\left(1-
\begin{bmatrix}
\lambda_k \\
\lambda_v
\end{bmatrix}\right)
\begin{bmatrix}
K \\
V
\end{bmatrix}+\mathcal{C}
\begin{bmatrix}
\lambda_k \\
\lambda_v
\end{bmatrix}
\begin{bmatrix}
U_k \\
U_v
\end{bmatrix}, \label{eq:1}\\
\begin{bmatrix}
\lambda_k \\
\lambda_v
\end{bmatrix}=\sigma\left(
\begin{bmatrix}
K \\
V
\end{bmatrix}
\begin{bmatrix}
W_{k_1} \\
W_{v_1}
\end{bmatrix}+\mathcal{C}
\begin{bmatrix}
U_k \\
U_v
\end{bmatrix}
\begin{bmatrix}
W_{k_2} \\
W_{v_2}
\end{bmatrix}\right),\label{eq:2}
\\
\mathcal{V}_{con} = \mathrm{ContextSA}(\mathcal{T},\mathcal{V}_{align}),
\\
\mathcal{A}_{con} = \mathrm{ContextSA}(\mathcal{T},\mathcal{A}_{align}),
\end{gather}
where $\{U_k,U_v\}\in\mathbb R^{d_c\times d}$ and $\{W_{k_1},W_{k_2},W_{v_1},W_{v_2}\}$ $\in\mathbb R^{d\times 1}$ are learnable parameters. $\mathcal{C}$ is contextual information. $\mathrm{ContextSA}(A,B)$ denotes computing self-attention using $A$ as Q/K/V while incorporating $B$ as contextual information through Eqs. \eqref{eq:1}, \eqref{eq:2}.

\textit{\textbf{Reciprocal Contextual Augmentation.}} Although the current features have achieved sufficient cross-modal interaction through text-centric contextualization, the direct interplay between visual and acoustic modalities remains underexplored, as both are merely appended as auxiliary contexts to the text representations. 
To address the lack of cross-modal synergy, we introduce reciprocal contextual augmentation. 
Specifically, we alternately inject the complementary modality %
as dynamic contexts in cross-attention layers to achieve deeper fusion:
\begin{gather}
\mathcal{V}_{out} = \mathrm{ContextCA}(\mathcal{T},\mathcal{V}_{con},\mathcal{A}_{con}),
\\
\mathcal{A}_{out} = \mathrm{ContextCA}(\mathcal{T},\mathcal{A}_{con},\mathcal{V}_{con}),
\end{gather}
where $\mathrm{ContextCA}(A,B,C)$ computes cross-attention using $A$ as Q, $B$ as K/V, and $C$ as context via Eqs. \eqref{eq:1}, \eqref{eq:2}. 

\textit{\textbf{Feature Fusion.}}
Finally, the obtained features are fused by a gated linear layer:
\begin{gather}
w_v =\sigma([\mathcal{T},\mathcal{V}_{out}]W_v+b_v), \\ w_a = \sigma([\mathcal{T},\mathcal{A}_{out}]W_a+b_a),\\
\mathcal{M}=\mathcal{T}+w_v\odot\mathcal{V}_{out}+w_a\odot\mathcal{A}_{out}.
\end{gather}
where $\{W_{v},W_{a},b_{v},b_{a}\}$ are learnable parameters.

\subsection{Causal Explanation Generation}
\label{3.5}

For the causal explanation generation pathway $\mathcal{M} \rightarrow \mathcal{E}$, we take advantage of the BART decoder due to its inherent suitability for conditional text generation. 
Given the multi-modal fused features $\mathcal{M}$, the explanation $\mathcal{E}$ is directly generated through:
\begin{equation}
\mathcal{E}=\mathrm{Decoder}(\mathcal{M}).
\end{equation}
Following Eq. \eqref{eq:elo}, the explanation generation loss is formulated as:
\begin{equation}
\mathcal{L}_{exp}=-\frac{1}{N}\sum_{i=1}^N e'_i \log e_i, \label{eq:exploss}
\end{equation}
where $N$ denotes the total tokens in the ground-truth $\mathcal{E}'$, $e'_i \in \mathcal{E}'$ represents the $i$-th ground-truth explanation token, and $e_i \in \mathcal{E}$ corresponds to the predicted token probability.

\begin{table*}[t]
\centering

\caption{Main Results on MUStARD dataset. * indicates our reproduced results. The p-value of the significance test between \textbf{BOLD} result of MuVaC and the corresponding result of VyAnG-Net is less than 0.01.}
\label{tab:main}

\begin{tabular}{clcccccccc}
\Xhline{1.1pt}
\multirow{2}{*}{\textbf{Modality}} & \multirow{2}{*}{\textbf{Method}}                                                      & \multicolumn{4}{c}{\textbf{Speaker Dependent}} & \multicolumn{4}{c}{\textbf{Speaker Independent}} \\ \cline{3-10} 
                                   &                                                                                       & Acc.(\%)   & Pre.(\%)   & Rec.(\%)   & F1(\%)  & Acc.(\%)    & Pre.(\%)   & Rec.(\%)   & F1(\%)   \\ \hline
\multirow{3}{*}{T}                 & BERT (\citeyear{Devlin_Chang_Lee_Toutanova_2019})                     & 64.3       & 65.6       & 64.3       & 64.7    & 58.0        & 58.2       & 56.7       & 57.4     \\
                                   & SMSD (\citeyear{10.1145/3308558.3313735})               & -          & 61.6       & 61.0       & 61.1    & -           & 51.7       & 48.2       & 47.0     \\ 
                                   & MIARN (\citeyear{tay-etal-2018-reasoning})     & -          & 64.7       & 64.0       & 63.9    & -           & 60.4       & 55.2       & 54.0     \\ \hline
T+V                                & FiLM  (\citeyear{10.1007/978-3-030-92307-5_21})                                                                               & 67.0       & 67.3       & 66.2       & 66.7    & 60.6        & 60.8       & 59.4       & 60.1     \\ \hline
\multirow{9}{*}{T+V+A}             & SVM (\citeyear{castro-etal-2019-towards})   & -          & 72.6       & 71.6       & 71.6    & -           & 64.3       & 62.6       & 62.8     \\
                                   & A-MTL (\citeyear{chauhan-etal-2020-sentiment})                            & -          & 73.4       & 72.8       & 72.6    & -           & 69.5       & 66.0       & 65.9     \\
                                   & QPM (\citeyear{Liu_Zhang_Li_Wang_Song_2021})                        & 77.5       & 77.5       & 77.6       & 77.5    & 66.2        & -          & -          & 66.1     \\
                                   & IWAN (\citeyear{Wu2021ModelingIB})                                       & -          & 75.2       & 75.2       & 75.1    & -           & 71.9       & 71.3       & 70.0     \\
                                   & HKT (\citeyear{Hasan_Lee_Rahman_Zadeh_Mihalcea_Morency_Hoque_2022}) &     73.6      &   73.6          & 73.6         &  73.6       &    61.6         &     63.9       &  61.6          &    61.7      \\
                                   & MO-Sarcation* (\citeyear{Tomar2023YourTS})                                  &  79.7          &   79.7         &  79.7         &  79.7       &  71.3       &    65.1        &  71.1     &  67.9        \\
                                   & VyAnG-Net  (\citeyear{pandey2024vyang})                                  & 79.9       & 78.8       & 78.2       & 78.5    & 76.9        & \textbf{75.7}       & 75.5       & \textbf{75.6}     \\
                                   & CESDN (\citeyear{li2024attention})                                       & 75.5       & 76.2       & 74.2       & 75.2    & 71.3        & 71.9       & 69.9       & 70.9     \\

                                   &MV-BART (\citeyear{zhuang2024mv}) &-&80.6 &82.8 &81.7&-&-&-&-\\
                                   &CCG-Net (\citeyear{zhuang2025cross})&-&79.1&79.1& 79.0&-&-&-&-\\ \cline{2-10}
                                   & \textbf{MuVaC (ours) }                                  
                                             &    \textbf{86.9}$^{\uparrow 7.0} $&\textbf{86.8}$^{\uparrow 6.2} $         &\textbf{89.2}$^{\uparrow 6.4} $            & \textbf{88.0}$^{\uparrow 6.3} $       &\textbf{77.4}$^{\uparrow 0.5} $             & 70.8           & \textbf{79.0}$^{\uparrow 3.5} $           & 74.7          \\ 
\Xhline{1.1pt}
\end{tabular}
\end{table*}

\subsection{Causal Feature Extraction}
\label{3.6}

Before sarcasm detection, it is necessary to obtain a robust causal feature $\mathcal{F}$ through the $\mathcal{E} \rightarrow \mathcal{F}$ path. 
We use a bidirectional Transformer to extract preliminary features from the explanation text:
\begin{equation}
 \mathcal{F}=\mathrm{BiTrans}(emb(\mathcal{E})),
\end{equation}
where $emb(\cdot)$ denotes the token embedding layer. 
To model robust explanation features, we assume that $\mathcal{F}$ follows a normal distribution $\mathcal{N}(\boldsymbol{\mu},diag(\boldsymbol{\sigma}^2))$. 
The mean vector $\boldsymbol{\mu}$ and variance $\boldsymbol{\sigma}^2$ %
are predicted via MLPs using the \verb|[CLS]| token feature $\mathcal{F}$:
\begin{equation}
\boldsymbol{\mu} =  \mathrm{MLP}_\mu(\mathcal{F}_{cls}), \; \log \boldsymbol{\sigma}^2=\mathrm{MLP}_\sigma(\mathcal{F}_{cls}),
\end{equation}
which ensures differentiability for optimizing Eq. \eqref{eq:clo}. %
To enforce consistency along the $\mathcal{E} \rightarrow \mathcal{F}$ pathway, both the training and testing phases use the generated explanation $\hat{\mathcal{E}}$. 
The training process is regularized by the KL divergence terms in Eq. \eqref{eq:clo} and Eq. \eqref{eq:elo}, which constrain explanation generation.
However, similar to humans who may either fail to recognize sarcasm without explicit cues or overgeneralize sarcasm to all dialogues, we introduce a small amount of causal intervention during training to mitigate such biases.
We intervene on $\mathcal{E}$ by probabilistically masking the $\mathcal{M} \rightarrow \mathcal{E}$ path. 
Specifically, we replace generated explanations $\hat{\mathcal{E}}$ with ground-truth $\mathcal{E}'$ for masked samples during training:
\begin{equation}
   \mathcal{F}=
\begin{cases}
\mathrm{BiTrans}(emb(\hat{\mathcal{E}})),  &\text{with }p=1-\varepsilon,    \\
\mathrm{BiTrans}(emb(\mathcal{E}')), &\text{with }p=\varepsilon,
\end{cases}
\end{equation}
where $\varepsilon$ is the hyperparameter. 
The basic model, which completes the $\mathcal{E} \rightarrow \mathcal{F}$ path using ground truth $\mathcal{E}'$, is consistent with this approach when we intervene on all samples~($\varepsilon=1$).

\subsection{Causal-Informed Sarcasm Detection}
\label{3.7}

After obtaining $\mathcal{F}$, we can complete the final $\mathcal{F}\rightarrow \mathcal{Y} \leftarrow \mathcal{M}$ path to obtain the result of sarcasm detection. 
We use the explanation feature $\mathcal{F}$ and the multimodal feature $\mathcal{M}$ to calculate the result:
\begin{equation}
\begin{aligned}
\hat{\mathcal{Y}}&=%
\mathrm{softmax}([\mathbb{E}(\mathcal{F});\mathcal{\bar{M}}]W_y+b_y),\\
&=\mathrm{softmax}([\boldsymbol{\mu};\mathcal{\bar{M}}]W_y+b_y),
\end{aligned}
\end{equation}
where $\mathcal{\bar{M}}$ is the average result of the pooling of $\mathcal{M}$.

\subsection{Optimization}
\label{3.8}

\label{sec:op}
The training objective comprises two components: $\mathcal{L}_{cls}$ and  $\mathcal{L}_{exp}$. 
$\mathcal{L}_{exp}$ is formalized in Eq. \eqref{eq:exploss}, $\mathcal{L}_{cls}$ is defined as:
\begin{equation}
\resizebox{0.9\linewidth}{!}{$
\begin{aligned}
\mathcal{L}_{cls}&=-\mathbb{E}_{q(\mathcal{F}|\mathcal{E}')}[\operatorname{log}p(\mathcal{Y}'|\mathcal{M},\mathcal{F})]+KL(q(\mathcal{F}|\mathcal{E'})\parallel p(\mathcal{F}|\hat{\mathcal{E}}))\\
&= -\mathbb{E}_{q(\mathcal{F}|\mathcal{E}')}[\operatorname{log}p(\mathcal{Y}'|\mathcal{M},\mathcal{F})]\\ 
&+\frac{1}{2}\left[tr\{\boldsymbol{\Sigma}_{\hat{\mathcal{E}}}^{-1}\boldsymbol{\Sigma}_{\mathcal{E}^{\prime}}\}+\Delta\boldsymbol{\mu}^\top\boldsymbol{\Sigma}_{\hat{\mathcal{E}}}^{-1}\Delta\boldsymbol{\mu}-d_{\mathcal{F}}+\log\frac{|\boldsymbol{\Sigma}_{\hat{\mathcal{E}}}|}{|\boldsymbol{\Sigma}_{\mathcal{E}^{\prime}}|}\right],
\end{aligned}
$}
\end{equation}
where the KL divergence term is the analytical solution under the normal distribution. $d_{\mathcal{F}}$ denotes the feature dimension of $\mathcal{F}$, $\Delta\boldsymbol{\mu}=(\boldsymbol{\mu}_{\hat{\mathcal{E}}}-\boldsymbol{\mu}_{\mathcal{E}'})$ and $\boldsymbol{\Sigma}=diag(\boldsymbol{\sigma}^2)$.
Directly computing the expectation $\mathbb{E}_{q(\mathcal{F}|\mathcal{E}')}$ is intractable, we adopt the Monte Carlo estimation following Xue et al. (\citeyear{xue2023variational}):
\begin{equation}
\resizebox{0.9\linewidth}{!}{$
-\mathbb{E}_{q(\mathcal{F}|\mathcal{E}')}[\operatorname{log}p(\mathcal{Y}'|\mathcal{M},\mathcal{F})]\approx -\frac{1}{H}\sum_{i=1}^H\mathcal{Y}'\log p(\mathcal{E}|\mathcal{M},\mathcal{F}_i),$}
\end{equation}
where $\mathcal{F}_i \sim \mathcal{N}(\boldsymbol{\mu}_{\mathcal{E}'},diag(\boldsymbol{\sigma}_{\mathcal{E}'}^2))$ denotes the $i$-th sample of the same distribution. 
Since direct sampling would block gradient backpropagation, we imitate the reparameterization process of the variational autoencoder and sample from the distribution with $\boldsymbol{\mu}_{\mathcal{E}'}$ and $\boldsymbol{\sigma}_{\mathcal{E}'}$:
\begin{equation}
\mathcal{F}_i = \boldsymbol{\mu}_{\mathcal{E}'}+\epsilon\odot\boldsymbol{\sigma}_{\mathcal{E}'},
\end{equation}
where $\epsilon\in\mathcal{N}(0,1)$ is a randomly sampled value.
The total loss is:
\begin{equation}
\mathcal{L}=\mathcal{L}_{cls}+\mathcal{L}_{exp}.
\end{equation}

\section{Experiments}
\subsection{Experimental Settings}
\quad \textit{\textbf{Dataset}}.
Our experiments utilize three benchmark datasets for multimodal sarcasm analysis. 
For the MSD task, we employ two benchmark datasets in multimodal sarcasm detection: MUStARD \cite{castro-etal-2019-towards} and MUStARD++ \cite{ray-etal-2022-multimodal}.
MUStARD contains 690 TV show situational dialogues, of which 345 are sarcastic and 345 are non-sarcastic.
Each dialogue is paired with video clips and corresponding transcripts. 
Following the initial split of the dataset, we evaluate both speaker-dependent and speaker-independent scenarios.
In the speaker-dependent setting, a five-fold cross-validation method is used; in the speaker-independent setting, the dialogues of \textit{Friends} drama are used as the test set, and others are the training set. 
MUStARD++ is a modification and extension of the MUStARD dataset, containing 1202 instances out of which 601 are sarcastic and 601 are non-sarcastic. 
We follow the dataset split of \citet{ray-etal-2022-multimodal}. 
For the MuSE task, we use the WITS dataset \cite{desai2022nice}, containing 2,240 sarcastic dialogues split into 1,792 training, 224 validation, and 224 test samples.
To enable explanation generation in MUStARD/MUStARD++, we leverage ChatGPT-4o to produce explanations.

\textit{\textbf{Baslines for MSD.}}
We select the following representative baseline models: \textbf{(1) Unimodal methods:} BERT \citep{Devlin_Chang_Lee_Toutanova_2019}, SMSD \citep{10.1145/3308558.3313735} and MIARN \citep{tay-etal-2018-reasoning}. 
\textbf{(2) Multimodal methods:} FiLM \citep{10.1007/978-3-030-92307-5_21}, SVM \citep{castro-etal-2019-towards}, A-MTL \citep{chauhan-etal-2020-sentiment}, QPM~\citep{Liu_Zhang_Li_Wang_Song_2021}, IWAN \citep{Wu2021ModelingIB}, HKT \citep{Hasan_Lee_Rahman_Zadeh_Mihalcea_Morency_Hoque_2022}, MO-Sarcation \citep{Tomar2023YourTS}, VyAnG-Net \citep{pandey2024vyang}, CESDN \citep{li2024attention}, MV-BART \citep{zhuang2024mv}, CCG-Net \citep{zhuang2025cross}, Collaborative Gating \citep{ray-etal-2022-multimodal}, AttFusion \citep{gao2024improving}, MEF \citep{shi2024multimodal}, %
and ViFi-CLIP~\citep{bhosale2023sarcasm}.

\textit{\textbf{Baslines for MuSE.}}
We compare the following sequence-to-sequence methods: RNN \citep{schuster1997bidirectional}, %
mBART \citep{liu2020multilingual}, BART \citep{lewis2019bart}, MAF \citep{kumar2022did}, MOSES \cite{kumar2023explaining}, EDGE \citep{ouyang2025sentiment} and CCG-Net \citep{zhuang2025cross}.

\textit{\textbf{Metrics.}}
For the MSD task, we use accuracy, weighted-precision, weighted-recall, and weighted-F1 as evaluation indicators \citep{Tomar2023YourTS}. 
For the MuSE task, we use the common metrics ROUGE-1/2/L \citep{Lin_2004}, BLEU-1/2/3/4 \citep{Papineni_Roukos_Ward_Zhu_2001}, METEOR \citep{Denkowski_Lavie_2014} and BERTscore \citep{bert-score}.

\textit{\textbf{Implementation details.}}
We use BART\footnote{\url{https://huggingface.co/facebook/bart-base}}, CLIP\footnote{\url{https://huggingface.co/openai/clip-vit-base-patch32}}, and CLAP\footnote{\url{https://huggingface.co/laion/clap-htsat-unfused}}~as encoders for the textual, visual, and auditory modalities respectively. 
We set the feature dimension $d=768$. ATF module is inserted at the 6-th layer of the BART encoder, with the causal intervention probability $\varepsilon$ set to 0.1. We employ the Adam optimizer. All experiments are conducted 5 times using a single RTX 4090 (24GB).

\subsection{Main Results}
\quad \textit{\textbf{MSD Performance.}}
Tables \ref{tab:main} and \ref{tab:mplus} present experimental results on MUStARD and MUStARD++ datasets. These results demonstrate the effectiveness of our proposed MuVaC from two perspectives: 
(1) MuVaC significantly outperforms strong baselines in speaker-dependent settings. 
Compared with MV-BART, it achieves improvements exceeding 6\% across all metrics. 
Besides, MuVaC demonstrates even stronger performance on MUStARD++, with an approximately 10\% F1 improvement compared to baselines. 
This highlights that MuVaC effectively leverages causal dependencies in explanations to predict more accurate answers, confirming the critical role of explanation information. 
(2) MuVaC has competitive but partially suboptimal performance in speaker-independent settings.
MuVaC outperforms VyAnG-Net in accuracy and recall but slightly lags in precision and F1. 
The result aligns with expectations: when tested on unseen speakers (e.g., "Chandler mocks Moderator" without prior exposure to either "Chandler" or "Moderator"), generative models fail to produce accurate explanations. %
This indirectly degrades prediction quality.
Despite this, MuVaC achieves higher recall and accuracy than all baselines, indicating its strong sensitivity to sarcastic dialogues.

\begin{table}[t]
\centering
\caption{Main results on MUStARD++ dataset.}
\label{tab:mplus}
\begin{tabular}{lccc}
\Xhline{1.1pt}
\textbf{Method}      & \textbf{Pre.(\%)} & \textbf{Rec.(\%)} & \textbf{F1(\%)} \\ \hline
SVM                  & 66.2              & 66.3              & 66.2            \\
Collaborative Gating & 69.8              & 69.5              & 69.5            \\
AttFusion            & 74.3              & 74.3              & 74.3            \\
VyAnG-Net            & 72.4              & 72.1              & 72.2            \\
MEF                  & 74.7              & 76.1              & 74.5            \\
Gaze                 & 73.2              & 73.2              & 73.3            \\
ViFi-CLIP            & 73.5              & 72.8              & 73.1            \\ \hline

\textbf{MuVaC}               & \textbf{81.2}                  & \textbf{85.3}                  &  \textbf{83.2}               \\ \Xhline{1.1pt}

\end{tabular}
\end{table}

\begin{figure*}[t]
    \centering
    \includegraphics[width=0.88\linewidth]{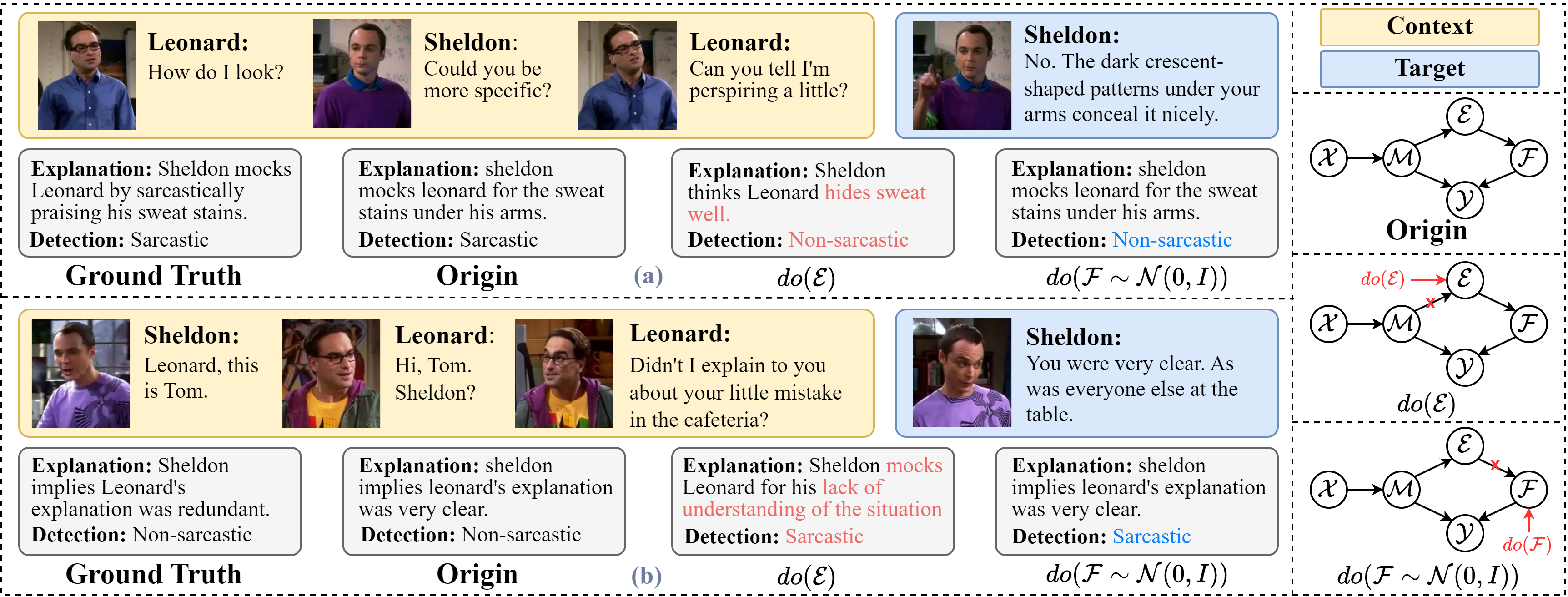}
    \caption{Causal results of manually intervening $\mathcal{E}$ and $\mathcal{F}$.}
    \label{fig:ce}
\end{figure*}

\textit{\textbf{MuSE Performance.}}
Table \ref{tab:exp} presents the results for sarcasm explanation generation on the WITS dataset. 
The findings show that MuVaC surpasses all baseline models on most evaluation metrics, including a ROUGE-1 score of 44.7 and a BLEU-1 score of 41.1.
Furthermore, MuVaC performs comparably with EDGE on semantic similarity metrics like METEOR and BERTscore (39.4 vs 39.9 and 79.3 vs 80.2, respectively).
This demonstrates that our causal modeling and ATF fusion mechanism can effectively integrate multimodal information, generating more well-structured explanations.

\begin{table}[t]
\centering

\caption{Main results on WITS dataset. (R1/2/L, B1/2/3/4, M and BS are the abbreviations of ROUGE-1/2/L, BLEU-1/2/3/4, METEOR, and BERTscore, respectively.)}
\label{tab:exp}
\setlength{\tabcolsep}{1mm}
\begin{tabular}{lccccccccc}
\Xhline{1.1pt}
\textbf{Model}       & \textbf{R1} & \textbf{R2} & \textbf{RL} & \textbf{B1} & \textbf{B2} & \textbf{B3} & \textbf{B4} & \textbf{M} & \textbf{BS} \\ \hline
RNN         & 29.2        & 7.9         & 27.6        & 22.1        & 8.2         & 4.8         & 2.9         & 18.5       & 73.2        \\
mBART       & 33.7        & 11.0        & 31.5        & 22.9        & 10.6        & 6.1         & 3.4         & 21.0       & 76.0        \\
BART        & 36.9        & 11.9        & 33.5        & 27.4        & 12.2        & 6.0         & 2.9         & 26.9       & 73.8        \\
MAF         & 39.7        & 17.1        & 37.4        & 33.2        & 18.7        & 12.4        & 8.6         & 30.4       & 77.7        \\
MOSES       & 42.2        & 20.4        & 39.7        & 34.9        & 21.5        & 15.5        & 11.5        & 32.4       & 77.8        \\
EDGE        & 44.4        & 21.8        & \textbf{42.4}        & 37.6        & 23.2        & 16.6        & 12.9        & \textbf{39.9}       & \textbf{80.2}        \\
CCG-Net     & 42.8        & 21.7        & 39.7        & 35.1        & 22.6        & 16.2        & 11.3        & 33.5       & 77.3        \\ \hline
\textbf{MuVaC}      & \textbf{44.7}        & \textbf{22.1}        & \textbf{42.4}        & \textbf{41.1}        & \textbf{26.1}        & \textbf{17.9}        & \textbf{13.2}        & 39.4       & 79.3        \\ \Xhline{1.1pt}

\end{tabular}
\end{table}

\textit{\textbf{Multi-task Performance.}}
To further demonstrate MuVaC's multi-task superiority, we conduct comparative experiments on the annotated MUStARD dataset, with results shown in Table \ref{tab:mtask}. 
All single-task baselines implement multi-task learning by sharing features from the last layer. 
The experiments reveal: 
(1) MuVaC achieves comprehensive superiority on both MSD and MuSE tasks;
(2) although CCG-Net maintains balanced performance across tasks, its conventional framework with implicit feature sharing exhibits performance bottlenecks;
(3) the MSD method MO-Sarcation shows significant performance degradation on MuSE, exposing cross-task transfer limitations in single-task optimization. 
These findings suggest that MuVaC improves the performance of two tasks simultaneously by explicitly modeling the causal relationship.

Based on evaluation metrics for all tasks, we confidently conclude that MuVaC serves as the current SOTA solution for sarcasm detection and explanation.
Its success stems from the integration of variational causal inference with explanation-aware modeling, effectively bridging the gap between detection and~explanation.

\begin{table}[t]
\centering
\caption{Multi-task results on MUStARD dataset in speaker dependent setting.}
\label{tab:mtask}
\begin{tabular}{lcccccc}
\Xhline{1.1pt}
\multirow{2}{*}{\textbf{Method}} & \multicolumn{3}{c}{\textbf{MSD}}              & \multicolumn{3}{c}{\textbf{MuSE}}             \\ \cline{2-7} 
                                   & Pre.          & Rec.          & F1            & R1            & B1            & M             \\ \hline
BART                               & 70.0          & 75.7          & 72.7          & 27.3          & 20.7          & 22.1          \\
MO-Sarcation                       & 79.7          & 79.7          & 79.7          & 22.8          & 20.1          & 17.0          \\
MOSES                              & 73.1          & 81.0          & 76.9          & 35.6          & 29.6          & 28.2          \\
MAF                                & 77.7 & 75.6          & 76.7          & 36.3          & 30.1          & 30.4          \\
CCG-Net                            & 79.1          & 79.1          & 79.0          & 36.6          & 31.5          & 30.6          \\ \hline
\textbf{MuVaC}                     & \textbf{86.8}          & \textbf{89.2} & \textbf{88.0} & \textbf{38.4} & \textbf{32.4} & \textbf{32.7} \\ \Xhline{1.1pt}

\end{tabular}
\end{table}

\subsection{Ablation Study}
To validate the effectiveness of each module in MuVaC, we design two sets of ablation experiments, targeting variational causal modeling and multimodal feature utilization components, respectively.
The results are shown in Table \ref{tab:ab_c}.

\textit{\textbf{Variational Causal Modeling.}} 
We design four variant models that remove key components: w/o $\mathcal{E}$ (removing explanation generator, using direct $\mathcal{M}\rightarrow\mathcal{Y}$ path), only $\mathcal{E}$ (solely using $\mathcal{M}\rightarrow\mathcal{E}\rightarrow\mathcal{F}\rightarrow\mathcal{Y}$ path),
w/o $\mathcal{L}_{exp}$ (removing explanation loss), and w/o  $\mathcal{L}_{cls}$ (removing classification loss). 
Based on the results, we have the following findings:
(1) Simply relying on the explaination feature $\mathcal{F}$ or the multimodal feature $\mathcal{M}$ shows suboptimal performance, indicating that sarcastic explanations can enhance detection capabilities but require joint optimization with original multimodal features.
This verifies that jointly modeling MuSE and MSD using a variational causal inference framework outperforms existing methods that employ separate task-specific heads for different tasks.
(2) Removing task-specific losses severely degrades corresponding performance, confirming $\mathcal{L}_{cls}$ and $\mathcal{L}_{exp}$'s effectiveness for detection and explanation generation.

\textit{\textbf{Multimodal Feature Utilization.} }
We design five variants for comparison. w/o \textit{F\&P} means discarding expression and posture features, w/o \textit{ATF} denotes direct concatenation of multimodal feature, and w/ MOSES means using multimodal fusion module in MOSES~\cite{kumar2023explaining}.
The results show that:
(1) w/o \textit{F\&P} leads to a 5.3\% drop in F1-score on MUStARD, which validates the critical importance of these auxiliary non-verbal cues for capturing sarcastic signals.
(2) The incremental performance gains observed with the step-by-step addition of +\textit{Align} and +\textit{Context} effectively demonstrate that each component of the ATF module is necessary for achieving robust multimodal feature fusion
(3) The ATF module can learn more robust features than other existing feature fusion modules like MOSES.
The results verify that MuVaC’s multimodal feature extraction and fusion modules can effectively support causal modeling with better results.

\begin{figure}[t]
    \centering
    \includegraphics[width=\linewidth]{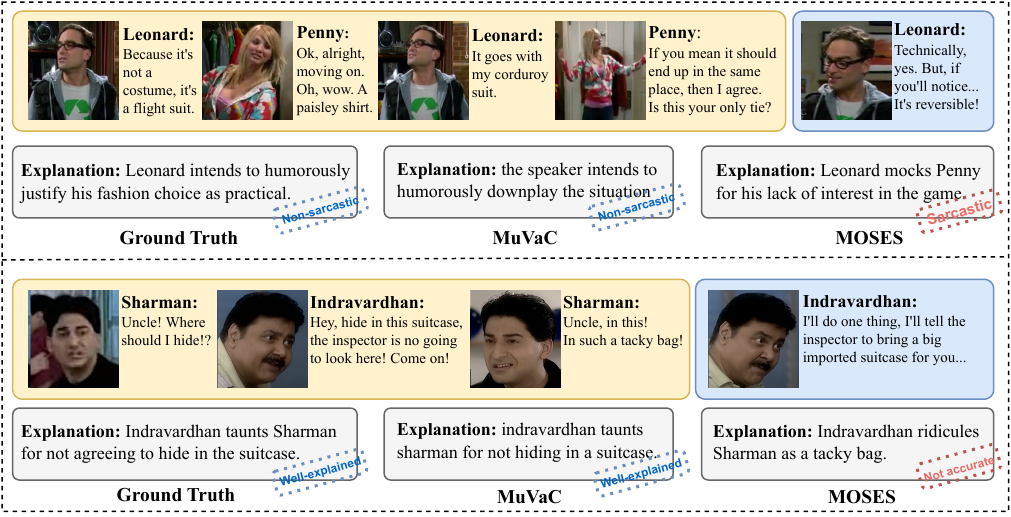}
    \caption{Case study on MUStARD and WITS datasets.}
    \label{fig:case}
\end{figure}

\begin{table}[t]
\centering
\caption{Ablation study of components (\%). }
\label{tab:ab_c}
\setlength{\tabcolsep}{1mm}
\begin{tabular}{lccccccccc}
\Xhline{1.1pt}
\multirow{2}{*}{\textbf{Variants}} & \multicolumn{3}{c}{\textbf{MUStARD}} & \multicolumn{3}{c}{\textbf{MUStARD++}} & \multicolumn{3}{c}{\textbf{WITS}} \\ \cline{2-10} 
                                   & Pre.    & Rec.   & F1   & Pre.     & Rec.    & F1    & R1          & B1          & M          \\ \hline
w/o $\mathcal{E}$                                &   84.2          &      86.5       & 85.3         &   76.4           &   85.3          &  80.6          &    -         &   -          &    -     \\ 
only $\mathcal{E}$                              &   79.5          &      83.8       &  81.6        &   75.0           &   83.6          &  79.1         &   42.4          &   38.4          &  37.2          \\ 
w/o $\mathcal{L}_{exp}$                                  &   83.8          &  83.8           & 83.8        &   76.4           &   80.3          &   78.4        &   0.1          &   1.2          &  1.1          \\
w/o  $\mathcal{L}_{cls}$                                &   62.0          &  48.6           &  54.5       &    65.5          &   62.3          &   63.9        &  42.6           &   38.4          &   36.3         \\ \hline
w/o \textit{F\&P}                                  &  81.6           &  83.8           &  82.7        &   78.9          &  \textbf{85.7}          & 82.2          &    42.3         &   38.2         &    35.9        \\   \hline
w/o \textit{ATF}                                  &  77.8           & 75.7           & 76.7         & 76.4             & 74.3            &  75.4        &   38.9          &   34.5           &   34.8       \\
\;+\textit{Align}                                &   82.9           & 78.4           & 80.6         &   77.8           &   80.0         &  78.9         &   40.6          &  37.1           & 35.4            \\
 \;\;+\textit{Context}                                &   85.3          &  78.4           &  81.7        &   80.5          &  82.8           &  81.7         &   41.9           & 37.9            &   35.8      \\
w/ MOSES                             &   82.8           & 82.6             &  82.6        &    79.3          &   79.3           &  79.3          &   42.1          &  37.8            &  36.4       \\
 
 \hline
\textbf{MuVaC}                              &    \textbf{86.8}         &    \textbf{89.2}         &  \textbf{88.0}        & \textbf{81.2}             &  85.3          &   \textbf{83.2}        & \textbf{44.7}           & \textbf{41.1}           & \textbf{39.4}        \\ \Xhline{1.1pt}
\end{tabular}
\end{table}

\subsection{Analysis of causal effects}
To investigate whether MuVaC learns the intended causal relationships, we analyze causal effects via manual causal interventions and conduct quantitative/qualitative evaluations.

\textit{\textbf{Quantitative Analysis.}}
We quantify causal effects by intervening on $\mathcal{E}$ and $\mathcal{F}$ during inferencing, and the results are shown in Table \ref{tab:ce}. 
$do(\mathcal{E})$ means using ground truth explanations $\mathcal{E}'$ instead of generated ones, blocking the $\mathcal{M}\rightarrow\mathcal{E}$ path; $do(\mathcal{F})$ replaces $\mathcal{F}$ with random Gaussian noise $\mathcal{F'}\sim \mathcal{N}(0,1)$, blocking the $\mathcal{E}\rightarrow\mathcal{F}$ path. 
The results show that when using ground truth $\mathcal{E}'$ directly, the indicators of the sarcasm detection approach nearly 100\%.
This demonstrates that MuVaC effectively models the causal relationship between MuSE and MSD, yielding consistent results.
In contrast, when intervening on $\mathcal{F}$, the performance becomes even worse than when no explanations are used at all, as shown in Table \ref{tab:ab_c}. 
In this case, the explanation and detection result are predicted independently, leading to inconsistent outcomes.

\begin{table}[t]
\centering
\caption{Analysis of causal effects (\%).}
\label{tab:ce}
\begin{tabular}{lccccccccc}
\Xhline{1.1pt}
\multirow{2}{*}{\textbf{Intervention}} & \multicolumn{3}{c}{\textbf{MUStARD}} & \multicolumn{3}{c}{\textbf{MUStARD++}} \\ \cline{2-7} 
                                   & Pre.    & Rec.   & F1   & Pre.     & Rec.    & F1          \\ \hline
$do(\mathcal{E})$                                  & 100.0           &  97.3           & 98.6        &  98.3           &   98.3           &  98.3         \\
$do(\mathcal{F})$                                & 76.3           &     78.4      &   77.3     &    75.0          &    78.7       & 76.8               \\ \hline
\textbf{MuVaC}                              &    86.8         &    89.2         &  88.0        & 81.2             &   85.3          &   83.2           \\ \Xhline{1.1pt}

\end{tabular}
\end{table}

\textit{\textbf{Qualitative Analysis.}}
We qualitatively analyze causal effects by observing the outcomes after manual intervention while keeping the input unchanged. 
As shown in Figure \ref{fig:ce}(a), MuVaC's original detection result and generated explanation highlight Sheldon's sarcastic response to Leonard.
When we manually replace the explanation with non-sarcastic content, the corresponding detection result changes to non-sarcastic, demonstrating the causal dependency of $\mathcal{Y}$ on $\mathcal{E}$. 
As depicted in Figure \ref{fig:ce}(b), when the explanation remains unchanged, replacing $\mathcal{F}$ with Gaussian noise leads to a non-sarcastic detection result. 
In this case, MuVaC cannot exploit the explanation information, resulting in independence between the explanation and the detection result.

Based on the above analysis, MuVaC establishes a reliable variational causal inference mechanism that ensures consistency between the model's explanation and detection result.

\subsection{Case Study}
To further verify the superiority of MuVaC, we design a case study on the MUStARD and WITS datasets. 
Since the WITS dataset is in Hindi, we provide its English translated version for consistent comparison, which is detailed in the Supplementary Materials.
Case 1 in Figure \ref{fig:case} shows the explanation information generated by MuVaC and other sarcastic explanation model MOSES on a non-sarcastic sample in MUStARD.
Only MuVaC correctly outputs the original meaning of the conversation, while MOSES believes that the conversation contains implicit sarcasm. 
This indicates that existing MuSE baselines struggle to grasp the true intentions of speakers in normal conversations. 
Case 2 presents a comparison of the explanations of sarcastic samples from the WITS dataset, where MuVaC outperforms the baselines in quality and is more helpful in reflecting sarcastic facts.
Therefore, MuVaC surpasses baselines in both correctness and quality.

\section{Conclusion}
In this paper, we introduce MuVaC, an innovative method that leverages variational causal inference to jointly integrate the MSD and MuSE tasks. 
This approach aims to simulate the human cognitive process of understanding sarcasm by explicitly modeling the causal relationship between sarcasm detection and explanation. 
To capture comprehensive sarcastic features, we not only incorporate multimodal cues often overlooked by existing models but also propose the ATF module for modality fusion. 
Experimental results demonstrate that MuVaC achieves state-of-the-art performance on both MSD and MuSE tasks, validating the causal relationship between tasks and offering a new perspective on multimodal sarcasm~understanding.

\begin{acks}
 This research is supported by the National Key R\&D Program of
 China (No. 2023YFC3303800).
 Xixun Lin is supported by the National Natural Science Foundation of China (No. 62402491) and the China Postdoctoral Science Foundation (No. 2025M771524). 
 Chuan Zhou is supported by the NSFC (No. 62472416). 
 Hao Peng is supported by the NSFC through grants U25B2029 and 62322202.
\end{acks}

\bibliographystyle{ACM-Reference-Format}
\bibliography{sample-base}
\appendix

\begin{table*}[t]
\centering
\caption{Comparison results with LLMs and MLLMs (\%). $^\dagger$ means fine-tuned on corresponding dataset.}
\label{tab:ab_llm}
\begin{tabular}{lcccccccccc}
\Xhline{1.1pt}
\multirow{2}{*}{\textbf{Methods}} & \multicolumn{3}{c}{\textbf{MuSTARD}} & \multicolumn{3}{c}{\textbf{MuSTARD++}} & \multicolumn{4}{c}{\textbf{WITS}} \\ \cline{2-11} 
                                  & Pre.       & Rec.       & F1         & Pre.        & Rec.        & F1         & R1     & B1     & M      & BS     \\ \hline
Llama 3.1                         & 53.8       & 83.1       & 65.3       & 50.3        & 99.3        & 66.8       & 20.5   & 18.8   & 18.3   & 72.9   \\
Qwen 2.5                          & 55.9       & 79.1       & 65.5       & 57.4        & 76.2        & 65.5       & 23.9   & 23.3   & 19.2   & 75.9   \\
MiniCPM-V 2.6                     & 57.2       & 62.6       & 59.8       & 52.4        & 73.7        & 61.3       & 20.4   & 17.9   & 18.8   & 71.8   \\
Qwen 2.5 VL                         & 61.9       & 41.4       & 49.6       & 63.2        & 46.2        & 53.4       & 21.8   & 20.7   & 18.5   & 74.0   \\
Kimi VL                            & 61.5       & 16.2       & 25.7       & 55.6        & 40.2        & 46.7       & 20.1   & 16.8   & 18.0   & 70.5   \\ \hline
Qwen 2.5 VL$^\dagger$                         &  67.5       & 70.3        &  68.9       &  74.4        &   55.8      &  63.8     & 24.7    & 26.0    & 23.2    & 77.9    \\ \hline
\textbf{MuVaC}                             & \textbf{86.8}       & \textbf{89.2}       & \textbf{88.0}       & \textbf{81.2}        & \textbf{85.3}       & \textbf{83.2}       & \textbf{44.7}           & \textbf{41.1}           & \textbf{39.4}  & \textbf{79.3}       \\ \Xhline{1.1pt}
\end{tabular}
\end{table*}

\section{More Experiments}

\subsection{Compared with LLMs and MLLMs}
To examine the capabilities of LLMs and MLLMs in sarcasm detection and explanation, we conduct comparative experiments, with results shown in Table \ref{tab:ab_llm}. 
Baseline models include Llama 3.1\footnote{\url{https://huggingface.co/meta-llama/Llama-3.1-8B-Instruct}}, Qwen 2.5\footnote{\url{https://huggingface.co/Qwen/Qwen2.5-7B-Instruct}}, 
Qwen2.5VL\footnote{\url{https://huggingface.co/Qwen/Qwen2.5-VL-7B-Instruct}}, KimiVL\footnote{\url{https://huggingface.co/moonshotai/Kimi-VL-A3B-Instruct}} and MiniCPM-V 2.6\footnote{\url{https://huggingface.co/openbmb/MiniCPM-V-2_6}}.
For sarcasm detection, we constrain the output to "yes" or "no". For sarcasm explanation, we adopt a 1-shot setting.

Experimental results indicate that LLMs and MLLMs, without targeted training and fine-tuning, struggle to achieve satisfactory performance on sarcasm detection and explanation, showing significant gaps compared with mainstream methods. Our analysis on the poor performance of LLMs/MLLMs is as follows. For the MSD task, we think the reason is that LLM/MLLMs cannot process the key tone information in dialogue scene. MLLMs mostly sample the video as images without considering the audio. For MuSE, the reason is that the form of the explanation generated by LLM/MLLMs is free. Even if the form is restricted in the instructions, it is difficult to perform well under token-level indicators like ROUGE. However, the MuSE performance of LLM/MLLM improves under semantic indicators like BERTscore.

\begin{table}[t]
\centering

\caption{Human evaluation (5$\times$224 samples, 0-5) and LLM-as-a-judge (0-5) on WITS.}
\label{tab:human}
\begin{tabular}{lcc}
\Xhline{1.1pt}
\textbf{Method} & \textbf{Human Evaluation} & \textbf{LLM-as-a-judge}\\ \hline
MOSES           & 2.51                      &3.24\\
MAF             & 2.54                      &3.28\\
EDGE            & 2.76                      &\textbf{3.36}\\
Qwen2.5VL    & \textbf{3.18}             &2.43\\
KimiVL      & 2.57                      &1.97\\ \hline
MuVaC           & \underline{2.82}          &\textbf{3.36}     \\ \Xhline{1.1pt}
\end{tabular}
\end{table}

\begin{table}[t]
\centering
\caption{Analysis of different sarcasm analysis task (\%).}
\label{tab:ab_task}
\begin{tabular}{lcc}
\Xhline{1.1pt}
  \textbf{Method}                      & \textbf{Matching} & \textbf{Ranking} \\ \hline
CLIP ViT-L               & 62.3     & 57.0    \\
T5-Large                & 59.6     & 61.8    \\
Qwen2.5VL-7b(zero-shot) & 50.1     & 52.1    \\
Qwen2.5VL-7b(5-shot)                & 53.7     & 53.8    \\
KimiVL-A3b(zero-shot)   & 58.6     & 55.2    \\
KimiVL-A3b(5-shot)                & 59.7     & 54.6    \\
ChatGPT 3.5(zero-shot)   & 50.4     & 52.8    \\
ChatGPT 3.5(5-shot)                & 63.8     & 55.6    \\\hline
\textbf{MuVaC}                   & \textbf{64.5}    & \textbf{65.5}   \\\Xhline{1.1pt}
\end{tabular}
\end{table}

\begin{table}[t]
\centering
\caption{Comparison results with different backbones (\%).}
\label{tab:backbone}
\resizebox{\linewidth}{!}{ 
\begin{tabular}{lccccccccc}
\Xhline{1.1pt}
\multirow{2}{*}{\textbf{Methods}} & \multicolumn{3}{c}{\textbf{MuSTARD}} & \multicolumn{3}{c}{\textbf{MuSTARD++}} & \multicolumn{3}{c}{\textbf{WITS}} \\ \cmidrule(l){2-10} 
                                   & Pre.       & Rec.       & F1         & Pre.        & Rec.        & F1         & R1        & B1        & M         \\ \hline

MuVaC(mBART)                       & 69.8       & 76.3       & 72.9       & 76.5        & 80.3        & 78.4       & 35.4      & 30.6      & 29.2      \\
MuVaC(BART)                        & 86.8       & 89.2       & 88.0       & 81.2        & 85.3        & 83.2       & 44.7           & 41.1           & 39.4      \\ 
\Xhline{1.1pt}
\end{tabular}
}
\end{table}

\subsection{Human evaluation and LLM as a judge}
Since the evaluation of sarcastic explanations is subjective, it is important to include human evaluation and LLM-as-a-judge evaluation. We recruit five volunteers to evaluate the sarcastic explanations on WITS (0-5, 5 is the highest), without knowing which model generated the explanations during the evaluation process. We use ChatGPT-4o as the evaluation model for LLM-as-a-judge.

The results are shown in Table \ref{tab:human}. From the results, we can conclude that MuVaC is competitive in human and LLM evaluation, second only to Qwen2.5VL-7b, a powerful MLLM.

\subsection{Analysis of different sarcasm analysis task}
MuVaC can handle new sarcastic forms after modification (from dialogue to image-text pair), as shown in Table \ref{tab:ab_task} for the results on HUB \cite{hessel-etal-2023-androids}. 
We select two new tasks, Sarcasm Matching and Ranking, as additional evaluations for MuVaC.
Matching requires models to select the finalist caption for the given cartoon from among distractors that were finalists, but for other contests. Quality ranking requires models to differentiate
a finalist from a non-finalist, both written for the given cartoon.
The prompt and settings are consistent with the original paper.
MuVaC performs better than zero-shot MLLM and fine-tuned CLIP. 
The results show that MuVaC can show adaptability in richer sarcasm understanding scenarios, demonstrating MuVaC's outstanding sarcasm understanding ability.

\subsection{Analysis of different backbone}
We conduct experiments to investigate the impact of different backbones, with results presented in Table \ref{tab:backbone}. The findings reveal that although mBART has stronger multilingual capabilities than BART, it does not demonstrate definitive performance improvements in sarcasm detection and explanation tasks.
At the same parameter scale, a larger semantic space does not learn more robust information. mBART's representation space is mixed. Its understanding of Hindi is inevitably interfered or averaged by its knowledge of dozens of other languages. This representation is good enough for a single task, but it may be noisy for joint tasks that require precise coordination. When MuVaC tries to optimize two objectives at the same time, this noise will be amplified.
In addition, existing work \cite{ouyang2025sentiment} also shows similar conclusions to ours.

\subsection{Analysis of different feature extractor}
To verify the compatibility of our method with different feature extraction approaches, we directly employed features extracted by MO-sarcasm and modified corresponding model components for comparison. As shown in Table \ref{tab:encoder}, MuVaC still demonstrates superior performance, achieving consistent improvements across different feature extraction methods.

\begin{table}[t]
\centering
\caption{Comparison results with features extracted by MO-sarcation \cite{Tomar2023YourTS} (\%).}
\label{tab:encoder}
\resizebox{\linewidth}{!}{ 
\begin{tabular}{lcccccccc}
\Xhline{1.1pt}
\multirow{2}{*}{\textbf{Methods}} & \multicolumn{4}{c}{\textbf{Speaker Dependent}} & \multicolumn{4}{c}{\textbf{Speaker Independent}} \\ \cmidrule(l){2-9} 
                                  & Acc.       & Pre.      & Rec.      & F1        & Acc.       & Pre.       & Rec.       & F1        \\ \hline
MO-Sarcation                      & 79.7       & 79.7      & 79.7      & 79.7      & 71.3       & 65.1       & 71.1       & 67.9      \\
CCG-Net                           & 78.2       & 79.4      & 72.9      & 76.1      & 64.0       & 57.1       & 63.1       & 60.0      \\
MuVaC                             & 82.6       & 81.1      & 81.1      & 81.1      & 76.4       & 69.7       & 78.9       & 74.0      \\ \Xhline{1.1pt}
\end{tabular}
}
\end{table}

\begin{figure}[tbp]
    \centering
    \includegraphics[width=\linewidth]{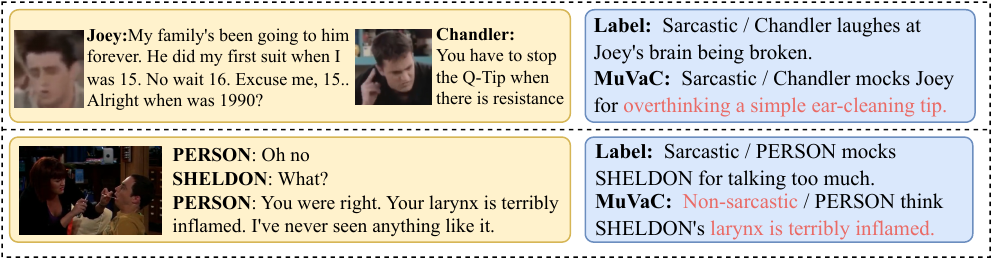}
    \caption{Failure cases.}
    \label{fig:failure}
\end{figure}

\begin{figure}[t]
    \centering
    \includegraphics[width=0.9\linewidth]{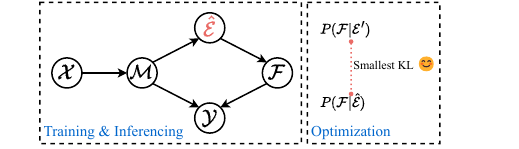}
    \caption{Variable causal inference pipline.}
    \label{fig:sup_vci}
\end{figure}

\subsection{Failure cases and limitations}
We provide two representative failure cases in Figure \ref{fig:failure}. 
The explanations in the failure cases tend to be superficial rather than having truly deep sarcastic connotations. 
The limitations mainly stem from the templated nature of the explanation generation, which may struggle to capture deeper sarcastic meanings. 
This may be caused by fixed training data formats and the backbone model's constraints. 
Future work could explore solutions using powerful LLM/MLLM.

\section{Related Derivations of Variational Causal Inference}

\subsection{Why $\log \boldsymbol{\sigma}^2$ instead of $\boldsymbol{\sigma}^2$ in Eq. 17?}
\;\;\; \textbf{Non-Negativity Constraint:} Variance must satisfy $\boldsymbol{\sigma}^2 > 0$. Directly predicting $\boldsymbol{\sigma}^2$ risks invalid negative outputs. By predicting log-variance ($\log \boldsymbol{\sigma}^2 \in \mathbb{R} $), we can ensure $\boldsymbol{\sigma}^2 = \exp(\log \boldsymbol{\sigma}^2)>0$.

\textbf{Numerical Stability:} Optimizing $\boldsymbol{\sigma}^2$ directly introduces instability near zero (e.g., exploding gradients). Using log-variance smooths gradient computation via the chain rule.

\subsection{KL Divergence Term in Eq.20}
\begin{proof}
    
The KL divergence term is the analytical solution under the normal distribution, which can be derived as follows:
\begin{equation}
\small
    \begin{aligned}
 & KL(q(\mathcal{F}|\mathcal{E}')\parallel p(\mathcal{F}|\hat{\mathcal{E}})) \\
 & =\int_{\mathcal{F}}q(\mathcal{F|\mathcal{E}'})\operatorname{log}\frac{q(\mathcal{F|\mathcal{E}'})}{p(\mathcal{F}|\hat{\mathcal{E}})}d\mathcal{F} \\
 &  =\int_{\mathcal{F}}\frac{1}{2}q(\mathcal{F}|\mathcal{E}')\Big[\operatorname{log}\frac{|\boldsymbol{\Sigma}_{\hat{\mathcal{E}}}|}{|\boldsymbol{\Sigma}_{\mathcal{E}^{'}}|}-(\mathcal{F}-\boldsymbol{\mu}_{\mathcal{E}^{'}})^{T}\boldsymbol{\Sigma}_{\mathcal{E}^{'}}^{-1}(\mathcal{F}-\boldsymbol{\mu}_{\mathcal{E}^{'}}) \\
 & \;+(\mathcal{F}-\boldsymbol{\mu}_{\hat{\mathcal{E}}})^T\boldsymbol{\Sigma}_{\hat{\mathcal{E}}}^{-1}(\mathcal{F}-\boldsymbol{\mu}_{\hat{\mathcal{E}}})\Big]d\mathcal{F} \\
  &= \frac{1}{2}\Big[\log\frac{|\boldsymbol{\Sigma}_{\hat{\mathcal{E}}}|}{|\boldsymbol{\Sigma}_{\mathcal{E}^{\prime}}|}-tr\big\{\mathbb{E}_{q}[(\mathcal{F}-\boldsymbol{\mu}_{\mathcal{E}^{\prime}})(\mathcal{F}-\boldsymbol{\mu}_{\mathcal{E}^{\prime}})^{T}]\boldsymbol{\Sigma}_{\mathcal{E}^{\prime}}^{-1}\big\} \\
  & \; +\mathbb{E}_q[(\mathcal{F}-\boldsymbol{\mu}_{\hat{\mathcal{E}}})^T\boldsymbol{\Sigma}_{E^*}^{-1}(\mathcal{F}-\boldsymbol{\mu}_{\hat{\mathcal{E}}})]\Big] \\
 & = \frac{1}{2}\Big[\log\frac{|\boldsymbol{\Sigma}_{\hat{\mathcal{E}}}|}{|\boldsymbol{\Sigma}_{\mathcal{E}'}|}-tr\{\boldsymbol{I}_{d_\mathcal{F}}\}+tr\{\boldsymbol{\Sigma}_{\hat{\mathcal{E}}}^{-1}\boldsymbol{\Sigma}_{\mathcal{E}^{\prime}}\}+\Delta\boldsymbol{\mu}^T\boldsymbol{\Sigma}_{\hat{\mathcal{E}}}^{-1}\Delta\boldsymbol{\mu}\Big] \\
 & = \frac{1}{2}\Big[\log\frac{|\boldsymbol{\Sigma}_{\hat{\mathcal{E}}}|}{|\boldsymbol{\Sigma}_{\mathcal{E}'}|}-d_\mathcal{F}+tr\{\boldsymbol{\Sigma}_{\hat{\mathcal{E}}}^{-1}\boldsymbol{\Sigma}_{\mathcal{E}^{\prime}}\}+\Delta\boldsymbol{\mu}^T\boldsymbol{\Sigma}_{\hat{\mathcal{E}}}^{-1}\Delta\boldsymbol{\mu}\Big],
\end{aligned}
\end{equation}

\end{proof}

\subsection{ELBO}

\begin{theorem}
    Suppose random vector $\mathcal{W} = (\mathcal{X},\mathcal{M},\mathcal{E},\mathcal{F},\mathcal{Y})$ follows a causal structure defined by the Bayesian network in Figure \ref{fig:sup_vci}. Then $\log p(\mathcal{Y}=\mathcal{Y'}|\mathcal{M})$ has the following variational lower bound:
    \begin{equation}
    \small
        \begin{aligned}
         &\log p(\mathcal{Y'}|\mathcal{M}) \\
         &\geq   \mathbb{E}_{q(\mathcal{F}|\mathcal{M})}[\operatorname{log}p(\mathcal{Y}'|\mathcal{M},\mathcal{F})]-KL(q(\mathcal{F}|\mathcal{M})\parallel p(\mathcal{F}|\mathcal{M})).
        \end{aligned}
    \end{equation}
    
\end{theorem}

\begin{proof}
The ELBO derivation process used in this article is as follows:
\begin{equation}
\small
\begin{aligned}
    &\log p(\mathcal{Y}'|\mathcal{M})\\
    &=\log \int_\mathcal{F}p(\mathcal{Y',F|M})d\mathcal{F}\\
    &=\log \int_\mathcal{F}p(\mathcal{Y',F|M})\frac{q(\mathcal{F|M})}{q(\mathcal{F|M})}d\mathcal{F}\\
    &=\log \mathbb{E}_{q(\mathcal{F|M})}\frac{p(\mathcal{Y',F|M})}{q(\mathcal{F|M})} 
    \\
    &\geq\mathbb{E}_{q(\mathcal{F|M})}\log \frac{p(\mathcal{Y',F|M})}{q(\mathcal{F|M})}  \;(\text{Jensen's Inequality}) \\
    &=\mathbb{E}_{q(\mathcal{F|M})}\log \frac{p(\mathcal{Y'|F,M})p(\mathcal{F|M})}{q(\mathcal{F|M})} \\
    &=\mathbb{E}_{q(\mathcal{F}|\mathcal{M})}[\log p(\mathcal{Y}'|\mathcal{F},\mathcal{M})+\log p(\mathcal{F}|\mathcal{M})-\log q(\mathcal{F}|\mathcal{M}) ] \\
    &=\mathbb{E}_{q(\mathcal{F}|\mathcal{M})}[\operatorname{log}p(\mathcal{Y}'|\mathcal{F},\mathcal{M})]-KL(q(\mathcal{F}|\mathcal{M})\parallel p(\mathcal{F}|\mathcal{M}))
\end{aligned}
\end{equation}

\end{proof}

\begin{algorithm}[htbp]
\caption{Video Keyframe Sampling and Feature Extraction}
\label{alg:keyframe}
\textbf{Input}: 
\begin{itemize}
    \item Video set $\mathcal{V}$ 
    \item Metadata JSON $\mathcal{J}$ containing \texttt{"utterance"} field
    \item Time weight $\alpha = 0.1$, target frames $k=100$, candidate frames $c=500$.
\end{itemize}
\textbf{Output}: $\mathbb{R}^{k \times d}$ tensor per video

\SetAlgoLined
\DontPrintSemicolon
Load CLIP model $\mathcal{M}$\;
Load metadata $\mathcal{D} \gets \textsc{ReadJSON}(\mathcal{J})$\;

\ForEach{video $v \in \mathcal{V}$}{
    $n_{\text{total}} \gets \textsc{GetTotalFrames}(v)$\;
    
    \eIf{$n_{\text{total}} < c$}{
        $\mathcal{I}_{\text{cand}} \gets \textsc{Oversample}(0, n_{\text{total}}-1, c)$ 
    }{
        $\mathcal{I}_{\text{cand}} \gets \textsc{LinSpace}(0, n_{\text{total}}-1, c)$ 
    }
    
    $\mathcal{F}_{\text{clip}} \gets \mathcal{M}.\textsc{Encode}(\textsc{ExtractFrames}(v, \mathcal{I}_{\text{cand}}))$ \;
    $\mathbf{T} \gets [\frac{i}{n_{\text{total}}} \mid i \in \mathcal{I}_{\text{cand}}]$ 
    }
    $\tilde{\mathcal{F}} \gets \mathcal{F}_{\text{clip}} + \alpha \cdot \mathbf{T}$ 
    
    $\{\mathcal{C}_1,...,\mathcal{C}_k\} \gets \textsc{K-Means}(\tilde{\mathcal{F}}, k)$ 
    $\mathcal{S} \gets \emptyset$ 
    
    \ForEach{cluster $\mathcal{C}_i \in \{\mathcal{C}_1,...,\mathcal{C}_k\}$}{
        $s^* \gets \arg\min_{s \in \mathcal{C}_i} \|s - \mu_i\|_2$ 
        $\mathcal{S} \gets \mathcal{S} \cup \{\mathcal{I}_{\text{cand}}[s^*]\}$ \;
    }
    
    $\mathcal{F}_{\text{out}} \gets \textsc{SortByTime}(\mathcal{F}_{\text{clip}}[\mathcal{S}])$ 

\end{algorithm}
\section{Differences from MOSES}
Both MuVaC and MOSES’ Context-Aware Attention are based on MO-Sarcation, which is a recognized contextual attention structure. However, there are two main differences: \textbf{(1) Different inputs.} The input of MuVaC is aligned features, while MOSES directly extracts features.\textbf{ (2) Different procedures.} MuVaC is text-oriented, considering the contextual interaction of visual and auditory modalities, and has low overhead (only 4 attention calculations). MOSES tediously considers the combination of different modalities, which has greater overhead (9 calculations) and suboptimal results.

\begin{table}[t]
\centering
\caption{Hyper-parameter settings.}
\label{tab:lr}
\begin{tabular}{lccc}
\Xhline{1.1pt}
                           & \textbf{MUStARD} & \textbf{MUStARD++} & \textbf{WITS} \\ \hline
base lr                    & 5e-6             & 6e-5             & 8e-06         \\
new lr                     & 18e-5            & 12e-5            & 2e-05         \\
insertion layer               & 6-th             & 6-th             & 6-th          \\
$\varepsilon$ & 0.1              & 0.1              & 0.1           \\ \Xhline{1.1pt}
\end{tabular}
\end{table}

\section{Prompts}
\quad $\verb|Prompt for generating explanation|$
\textit{For non-sarcastic samples.} Please analyze the following dialogue and explain why there is no sarcasm based on the context and utterance. Provide your answer in one sentence with a maximum of 15 words. For example: [EXAMPLE] The dialogue is as follows: [DIALOGUE].

\textit{For sarcastic samples. }
Please analyze the following dialogue and explain why there is sarcasm based on the context and utterance. Provide your answer in one sentence, including the source of sarcasm, the action, and the target, with a maximum of 10 words. For example: [EXAMPLE] The dialogue is as follows: [DIALOGUE].

$\verb|Prompt for MSD task.|$
I will give you a paragraph of dialogue. Please check if there is any sarcastic meaning in this text. You can only answer yes or no. The dialogue is below: [DIALOGUE].

$\verb|Prompt for MuSE task.|$ 
Please analyze the following dialogue and explain why there is sarcasm based on the context and utterance. Provide your answer in one sentence, including the source of sarcasm, the action, and the target, with a maximum of 10 words. The dialogue is below: [DIALOGUE].

$\verb|Prompt for LLM as a judge.|$
Please evaluate how well the predicted explanation matches the reference explanation for sarcasm. Rate from 0-5 where 5 means perfect match and 0 means completely different.

\section{Implementation Details}
We utilize BART\footnote{\url{https://huggingface.co/facebook/bart-base}}, CLIP\footnote{\url{https://huggingface.co/openai/clip-vit-base-patch32}}, and CLAP\footnote{\url{https://huggingface.co/laion/clap-htsat-unfused}} as encoders for the textual, visual, and auditory modalities respectively. 
We set the feature dimension $d=768$. The ATF layer is inserted at the 6-th layer of the BART encoder, with the causal intervention probability $\varepsilon$ set to 0.1. We employ the Adam optimizer. The hyper-parameters are summarized in Table \ref{tab:lr}. All experiments are conducted 5 times using a single RTX 4090 (24GB). 
The video frame sample algorithm is shown in Algorithm \ref{alg:keyframe}.
For expression features, we first use YOLOv8 \cite{Jocher_YOLO_by_Ultralytics_2023} for face detection, and then obtain the final features through CLIP after cropping.
For posture features, we use the paddle detection library to track character movement and obtain the key frames of each person’s movements. Then, we convert the character’s posture into a silhouette image and finally obtain the features through CLIP.

\end{document}